\documentclass[letterpaper, 10 pt, conference]{ieeeconf}  
\IEEEoverridecommandlockouts                              
\usepackage{mathrsfs}
\usepackage{color}
\usepackage{times}
\usepackage{epsfig}
\usepackage{graphicx}
\usepackage{tabularx}
\usepackage{amsmath}
\usepackage{amssymb}
\usepackage[ruled,vlined,linesnumbered]{algorithm2e}
\usepackage{graphicx} 
\usepackage{tikz}
\usepackage{verbatim}
\usepackage{booktabs}
\usepackage{multirow}
\usepackage{hyphenat} 
\usepackage{pgf}
\usetikzlibrary{arrows,automata}
\usepackage{url}
\usepackage{thmtools}
\usepackage{color}
\usepackage{times}
\usepackage{epsfig}
\usepackage{graphicx}
\usepackage{tabularx}
\usepackage{amsmath}
\usepackage{amssymb}
\usepackage[ruled,vlined,linesnumbered]{algorithm2e}
\usepackage{graphicx} 
\usepackage{tikz}
\usepackage{booktabs}
\usepackage{multirow}
\usepackage{makecell}
\usepackage{authblk}
\usepackage{hyphenat} 
\usepackage{pgf}
\usetikzlibrary{arrows,automata}

\usepackage[colorlinks=true]{hyperref}
\DeclareMathOperator*{\argmax}{arg\,max}

\title{\bf Anchor-Oriented Localized Voronoi Partitioning for GPS-denied Multi-Robot Coverage}

\author{Aiman Munir \and Ehsan Latif \and  Ramviyas Parasuraman
\thanks{The authors are with the School of Computing, University of Georgia, Athens, GA 30602, USA. Author Emails: {\tt\small \{aiman.munir,ehsan.latif,ramviyas\}@uga.edu}  
}}

\begin{document}

\newtheorem{definition}{Definition}
\newtheorem{theorem}{Theorem}
\newtheorem{lemma}{Lemma}
\newtheorem{proposition}{Proposition}
\newtheorem{property}{Property}
\newtheorem{observation}{Observation}
\newtheorem{corollary}{Corollary}

\maketitle
\thispagestyle{empty}
\pagestyle{empty}

\begin{abstract}
Multi-robot coverage is crucial in numerous applications, including environmental monitoring, search and rescue operations, and precision agriculture. In modern applications, a multi-robot team must collaboratively explore unknown spatial fields in GPS-denied and extreme environments where global localization is unavailable. Coverage algorithms typically assume that the robot positions and the coverage environment are defined in a global reference frame. However, coordinating robot motion and ensuring coverage of the shared convex workspace without global localization is challenging. This paper proposes a novel anchor-oriented coverage (AOC) approach to generate dynamic localized Voronoi partitions based around a common anchor position. We further propose a consensus-based coordination algorithm that achieves agreement on the coverage workspace around the anchor in the robots' relative frames of reference. 
Through extensive simulations and real-world experiments, we demonstrate that the proposed anchor-oriented approach using localized Voronoi partitioning performs as well as the state-of-the-art coverage controller using GPS. 
\end{abstract}

\section{Introduction}
\label{sec:intro}


Multi-robot systems (MRS) have recently attracted a lot of interest for various use cases involving significant cooperation among robots, such as logistics, surveillance, monitoring, target tracking, and search and rescue. 
Out of these potential use cases, multi-robot coverage is a well-studied problem, where the goal is to ensure robots' optimal deployment in an environment for persistently gathering important information \cite{chen2020distributed} or exploration \cite{latif2023seal}. 
To obtain optimal environment coverage with networked mobile robots, 
Voronoi partitioning 
offers a straightforward answer that ensures the networked robots will eventually converge to the configuration that maximizes environment coverage (i.e., centroidal Voronoi Tessellation \cite{cortes2004coverage}). 

There are numerous ways to create an exact Voronoi diagram, either by communicating with the robots and broadcasting messages or assuming that each robot can determine the positions of all the others \cite{guruprasad2012distributed,gosrich2022coverage}.
However, coverage is challenging for robots with limited resources and constraints, particularly in GPS-denied or extreme environments where global localization is unavailable or cannot be shared due to location-based privacy limitations.

\begin{figure}[t]
\centering
 \includegraphics[width=1\linewidth]{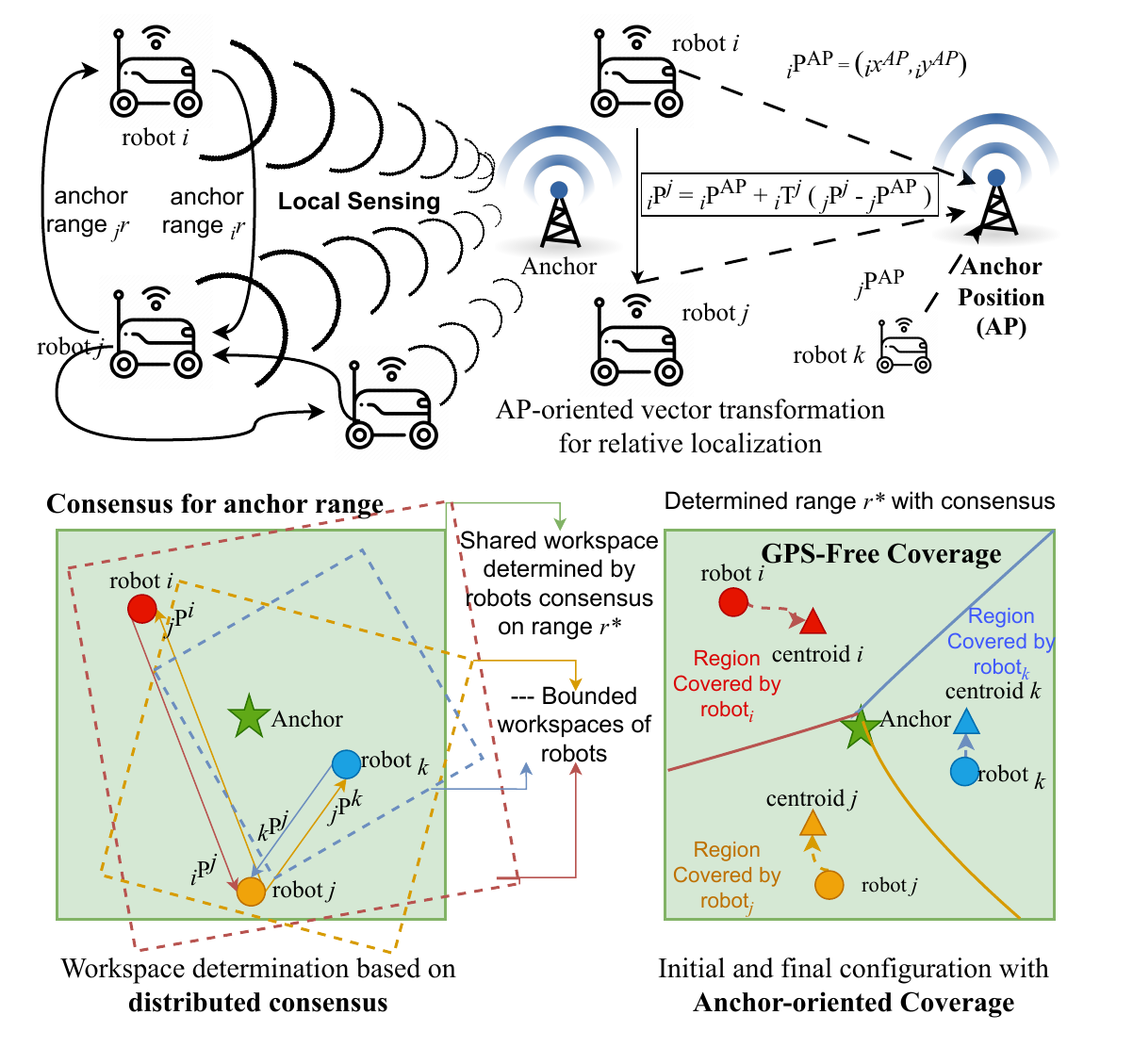}
 \caption{Overview of the proposed anchor-oriented multi-robot coverage, with consensus control on boundary size.}
 \label{fig:overview}
\end{figure}


Few approaches have attempted to overcome the implications of uncertain positioning with a known workspace, such as the distributed methods \cite{he2018distributed} to compute the Voronoi partitioning for a small sensor network with a fixed workspace or limited sensing range or using the Guaranteed Voronoi partitioning to deal with the uncertainty in the sensors \cite{papatheodorou2018distributed}. 
However, these works are based on global positioning; therefore, to compute each robot's Voronoi cell, the coverage method has to know the location of the Voronoi neighbors and the bounds of the convex environment in a global frame. 

We propose a new coverage law that can work without a global reference frame to address these challenges. The robots are equipped with sensors that can locate an anchor with some uncertainty and perform inexpensive computations to determine the relative distance of neighbor robots using the anchor as a reference point.
Yet, the robots need common coverage boundaries to partition their regions. A non-agreed workspace can lead to inconsistent coverage problems.
To address this issue, we propose a consensus-based law to help robots identify the coverage bounds of the convex workspace centered around the anchor point in their local reference frames and compute the necessary Voronoi partitioning in a fully distributed manner. 
Since the only information required by the control rule may be locally gathered by onboard sensors, the proposed coverage approach can be applied to a group of robots placed in an unknown environment. 
The high-level overview of the proposed approach is presented in Fig.~\ref{fig:overview}, depicting a GPS-free coverage with 3 robots.

The key contributions of the proposed approach are 
\begin{itemize}
  \item We propose a novel anchor-oriented coverage technique, which takes advantage of the fact that all robots may perceive the same anchor position in their local reference frames and perform optimal coverage of a workspace centered around the anchor. 
  \item A distributed consensus-based approach is leveraged to determine the workspace boundary in the robot's local frames of references, considering robots can estimate the anchor location with some uncertainty and share information with each other. The consensus achieved on the boundary radius and orientation around the anchor would then be used to construct workspaces in local reference frames, with each robot performing Voronoi partitioning locally, thus disregarding the necessity for a global coordinate system.
  \item Centroid seeking coverage law is derived for the localized Voronoi partitions with the agreed-upon local workspace bounds. We theoretically analyze the method and show that the proposed anchor-oriented Voronoi partitioning will provide a mathematical equivalent to global positioning-based partitioning. We also derive a theoretical upper bound for the cumulative regret of the locational cost achievable by the robots due to the influence of the uncertainty in anchor position estimates.
  \item We validate the proposed approach through extensive simulations and real-world robot demonstrations (see the attached video) and show that our approach, without using global coordinates, can match the performance of a GPS-aided coverage controller. 
  \item We open-source the relevant codes in Github\footnote{\url{https://github.com/herolab-uga/mr-aoc}} for further use and development by the community.
  \end{itemize}

\section{Related Work}
\label{sec:relatedwork}
In the literature, 
the early multi-robot coverage approaches proposed a distributed Lloyd's algorithm for coverage control in mobile sensor networks \cite{lloyd1982least,schwager2006distributed,gusrialdi2008Voronoi}. Using Lloyd's technique, these approaches partition the sensing region into Voronoi cells and allocate them to specific sensors for coverage with global positioning. 
These methods made strong assumptions, such as reliable robotic communication and known sensing capabilities, which made them impractical. To address this imperfect communication, the approaches \cite{miah2014nonuniform,siligardi2019} offer coverage control techniques for MRS with intermittent communication using precise localization and mapping information; however, these approaches cannot be applied to GPS-denied (or challenging indoor) environments where absolute location is unavailable. Also, sharing complete localization and sensing information among networks raises concerns regarding the location privacy of robots.




On the other hand, recent research such as \cite{pratissoli2022coverage, luo2019minimum} attempts to mitigate the effect of imprecise localization for multi-robot coverage. These approaches proposed optimal coverage by employing a distributed coordination technique and maintaining a minimal number of connectivity in multi-robot systems. To account for the uncertainty of the localization into the coverage controllers, research in \cite{papatheodorou2018distributed} used additively-weighted (power diagram) Guaranteed Voronoi partitioning, which assigns an importance weight based on the uncertain region to divide the environment into regions. However, if the weights of the power diagram are imbalanced, certain areas may suffer from under or over-coverage. These approaches assume the same sensing capabilities of robots and the known but uncertain localization information. 
Moreover, most approaches consider the coverage workspace clearly defined (or known in a global frame of reference).

Relative localization \cite{wanasinghe2015relative,faigl2013low,wiktor2020ICRA,latif2022multi} can be a possible solution in GPS-denied and extreme environments and enable the system to operate optimally if employed with correct measures. 
For instance, we proposed a graph-theoretic approach to relative localization in a previous study \cite{latif2022dgorl}. While it addressed some scalability issues of MRS, it still encounters high computational complexity and requires range (or RSSI) sensing between every robot combination, which may not be practical in highly resource-constrained systems \cite{faigl2013low}. Therefore, a system relying on a single environmental anchor like the Access Point of a wireless connection could be more practical. 
There have been a few attempts, such as \cite{pierson2017distributed,chen2022distributedMulti_target,pimenta2010simultaneous}, which consider relative localization for target tracking, but there is no such work exists which attempts to achieve optimal coverage with relative localization.

Our paper remedies these gaps by contributing a new coverage partitioning approach that uses detection and tracking of an anchor, which all robots can see in a team. Then, anchor-oriented partitioning is derived by sharing anchor positions in the robot's local frame of reference.
While performing localized Voronoi partitioning, it has become challenging to determine the boundary of the workspace because of the heterogeneous sensing capabilities of robots \cite{arslan2016voronoi,santos2018coverage,manjanna2018heterogeneous}. 
Therefore, departing from the extant frameworks for coverage control, we propose a new consensus-based approach to arrive at the decision of coverage workspace between multiple robots, considering the uncertainties in the anchor position estimates. 

\section{Multi-Robot Coverage}
\label{sec:coverage}

Let us consider $N$ number of robots defined by a set $R =\{1,2,\ldots,N\}$ are connected via a wireless network, forming a connected graph $\mathcal{G} = (R,E)$ defined by their edges $E = \{(i,j) | i \in \mathcal{N}_j \}$. $\mathcal{N}_j = \{i \in R | (i,j) \in E \}$ is the neighbor set of robot $i$ in the graph $\mathcal{G}$. The cardinality of the set $|\mathcal{N}_i| \in \mathbb{R}^n$ represents the number of neighbors for robot $i$.

Multi-robot coverage refers to strategically distributing robots across a domain to optimally cover the area of interest. 
Here, the set of $R$ robots is tasked with covering a convex environment optimally denoted as $\mathcal{Q} \subset \mathbb{R}^2$. We use ${\mathbf{p}}_i \in \mathcal{Q}$ to represent the position of robot $i$ and $q$ to denote a point within the workspace $Q$. Voronoi partition divides the workspace $Q$ into disjoint region of dominance for each robot $i$ based on its proximity to each point $q \in Q$, such that $\bigcup_{i=1}^{n} \mathcal{V}_i = Q$ and $\bigcap_{i=1}^{n} \mathcal{V}_i = \emptyset$.  This partition $\mathcal{V}_i$ for each robot $i \in R$ can be computed as follows:
\begin{equation}
\mathcal{V}_i=\left\{q \in \mathcal{Q} \mid\left\|q-\mathbf{p}_i\right\| \leq\left\|q-\mathbf{p}_j\right\|, \forall j \neq i, i \in R\right\} .
\label{eq: Voronoi}
\end{equation}
Here, $\| \cdot \|$ represents the Euclidean norm. The Voronoi partition would be referred to as Centroid Voronoi tessellation (CVT) \cite{du1999centroidal} when all the robots are at the centroid of their respective region of dominance. The cost function can determine the quality of coverage \cite{cortes2004coverage}, defined as:
\begin{equation}
H_{\mathcal{V}}(\mathbf{p}_1, ..., \mathbf{p}_n) = \sum_{i=1}^n \int_{\mathcal{V}_i} \frac{1}{2} |q - \mathbf{p}_i|^2 \phi(q) dq ,
\label{tarditionalLLoydCost}
\end{equation}
where $\phi(q)$ is a density function $\mathcal{Q} \to R>0$ to describe the importance of a given point q. 

Taking the gradient of the locational cost function yields the following control law that drives the robots to follow the direction of the negative gradient and moves them towards their respective centroid $C_{\mathcal{V}_i}$ (Lloyd's algorithm),
\begin{equation}
\dot{\mathbf{p}_i} = - k\left(C_{\mathcal{V}_i}-\mathbf{p}_i\right)
\label{eq: position controller}
\end{equation}
\begin{equation}
C_{\mathcal{V}_i} = \frac{1}{M_{\mathcal{V}_i}} \int_{\mathcal{V}_i} q \phi(q) dq  \text{ and } M_{\mathcal{V}_i}  = \int_{\mathcal{V}_i}  \phi(q) dq ,
\label{eq:centroid}
\end{equation}
where $k$ is a positive control gain.
Applying Lloyd's controller will yield an optimal distribution (partition) of the workspace as
$H_{\mathcal{V}}(\mathbf{p}_1^*, ..., \mathbf{p}_n^*)$, where $p_i^* = C_{\mathcal{V}_i}$.

Note that calculating the centroid of the Voronoi region in Eq.~\eqref{eq:centroid} needs two pieces of information: $\mathcal{V}_i \in \mathcal{Q}$ and $p_i \in \mathcal{Q}$, both of which are defined only in the global FoR (i.e., expressed in a world frame or GPS coordinate system).
\section{Anchor-oriented Coverage Approach}
\label{sec:method}

First, we formulate the GPS-denied coverage problem and then describe the proposed distributed controller for multi-robot coverage along with its theoretical properties.

\subsection{Localized Voronoi Partitioning Problem}
\label{sec:localizedvoronoi}
Let ${}^{i}{\mathbf{p}}_i$ represent the robot $i$'s position, which is available only in local ($i$'s internal) FoR. Here, the variable's prescript ${}^i$ denotes the FoR, and the subscript ${}^i$ denotes the concerning robot.
In a GPS-denied environment (or in situations where robots cannot share their information in a common coordinate system), there is no global FoR. Therefore, each robot has to perform the Voronoi partitioning and coverage in their internal FoRs, i.e., each robot operates with its workspace ${}^i \mathcal{Q} \in \mathbb{R}^2$. 
The problem here is to achieve optimal coverage (minimum cost in Eq.~\eqref{tarditionalLLoydCost}) in a distributed fashion, where each robot $i$ needs to perform the localized Voronoi partitioning of the multi-robot workspace 
without access to a global FoR.
To the best of the authors' knowledge, this unexplored problem has substantial practical implications in realizing truly distributed multi-robot deployments in applications like search and rescue or in extreme environments (e.g., extra-terrestrial) where infrastructure could be scarce. 


\begin{figure}[t]
\centering
 \includegraphics[width=0.98\linewidth,trim={4mm 4mm 4mm 4mm},clip]{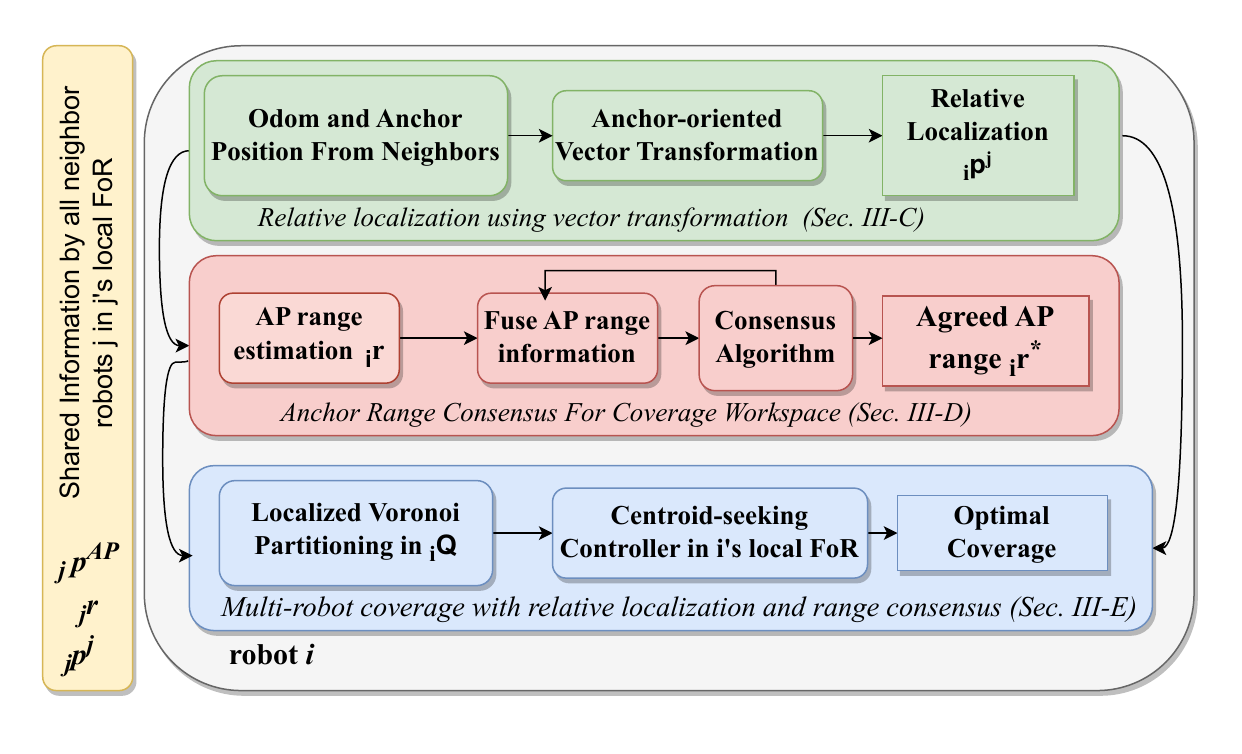}
 \vspace{-4mm}
 \caption{The distributed system architecture of multi-robot coverage for robot $i$ using anchor-oriented boundary consensus and partitioning to solve the GPS-denied coverage problem.}
 \label{architecture}
 \vspace{-4mm}
\end{figure}

We assume that all robots can locate an anchor (AP) \footnote{It is possible to extend this work to multiple anchor settings for large environments, as long as two neighboring robots can detect at least one common anchor in their vicinity.} and estimate the anchor's position with some uncertainty as a Gaussian distribution ${N}({}^i \mu_{AP}, {}^i \sigma^2_{AP})$.  
The anchor may refer to any one of the following: a known, unique feature in the environment, an Access Point of the wireless network, a common target for all robots, a source of an environmental function being measured (e.g., heat or radiation source), a powerful robot/satellite/drone which all other robots can see, etc. 
A Bayesian inferencing and filtering technique is assumed to be used to locate or track a landmark/object (e.g., locating Wi-Fi Access Point using Gaussian Process Regression approach \cite{latif2023gprl}).
Each robot is also equipped with an inexpensive, ubiquitously available Inertial Measurement Unit (IMU) for local odometry.
We address the above problem by proposing a new anchor-oriented coverage controller. In this controller, the anchor is used as the reference point for bounding the coverage workspace and performing Voronoi partitioning.
Fig.~\ref{architecture} delineates the distributed system architecture and summarizes the proposed approach. 
Note each robot locates the anchor using its onboard (local) sensor.

\subsection{Anchor Boundary Consensus}
\label{sec:anchor_boundary_consensus}
As noted in Eq.~\eqref{eq:centroid}, the workspace boundaries where the robots are to cover are unavailable in a global frame. Therefore, the robots have to agree on a common workspace boundary that they can apply in their local frames.
To address this problem, we use the commonly detected anchor as the center of a workspace (instead of global coordinate reference bounds). For this, robots need to determine and agree on the size of a squared boundary (width and orientation) centered around the anchor position. Each robot's optimal sensor/visibility range and its current orientation with respect to the anchor are used to determine the commonly agreed-upon boundary around the anchor. 
We refer to this region as a local environment for the robot, i.e., ${}^i\mathcal{Q} \subset \mathbb{R}^2$. 
Precisely, to estimate this local environment, a square boundary with the inscribed circle is used, with the circle radius being $ r^*$ and centered around the anchor ${}^i{}\mathbf{p}_{AP}$. i.e.,
\begin{eqnarray}
\begin{aligned}
 {}^i\mathcal{Q} & = \{ {}^i q = ({}^i x_q,{}^i y_q)  \mid \,  \left\|({}^i x_q - {}^i x_{AP}) + ({}^i y_q-{}^i y_{AP}) \right\| \\ 
 & + \left\|({}^i x_q - {}^i x_{AP}) - ({}^i y_q - {}^i y_{AP}) \right\| 
 \leq 2  r^* \, \},
\end{aligned}
\label{eq:square}
\end{eqnarray}


We will use a distributed consensus algorithm to determine this boundary. 
Once a robot $i$ advertises the tuple ${}^i b = ({}^i r, {}^i \theta)$ that includes the radius and orientation it wants to have, it broadcasts this to its local neighbors, and each robot then runs a distributed consensus algorithm such as \cite{parasuraman2019consensus} to determine the agreeable radius $r^*$ and orientation $\theta^*$ of the localized workspace centered around the anchor.

The consensus-based update rule for each robot $i$ is 
\begin{eqnarray}
{}^i b (k+1) = w_i(k)  {}^i b(k) + \alpha \sum_{j \in \mathcal{N}_i} w_j(k)  ({}^j b(k) - {}^i b(k)), 
\label{eq:consensus}
\end{eqnarray}
where ${}^i r(k)$ is the advertised radius by robot $i$ at time $k$, $w_{i}(k)$ is a weight coefficient that represents the robot i's uncertainty in estimating the anchor's location at instant $k$ (i.e., $w_{i}(k) \propto \frac{1}{{}^i \sigma^2_{AP}}$), and $\alpha$ is a constant update step size satisfying $0 < \alpha \leq 1$.
The motivation to achieve consensus on the width of the boundary is that the robots should rely more on the optimal sensing radius of the robot with the lowest uncertainty. For instance, assuming all robots carry different types of sensors to obtain the anchor location, the robot detecting the anchor with low uncertainty will influence its advertised radius to be used by the whole group.
Similarly, the motivation for the orientation consensus is to minimize the changes in turns each robot makes to achieve the coverage. Here, each robot will advertise its current orientation, and it will reach a consensus on an angle that provides the least disruption for the whole group. 

This distributed update with local neighbors is then propagated to the neighbor of neighbors until all the robots $i \in R$ of the multi-robot team in $\mathcal{G}$ reaches the same value $ b^* = {}^i b(k+\tau)$, where $\tau$ is the number of iterations (communication rounds) to achieve convergence.

From graph theory applied to multi-robot controllers \cite{zavlanos2011graph}, we know for a fact that the value ${}^i r \,\forall i \in R$ would eventually converge to a common value if $\mathcal{G}$ is connected. i.e., 
\begin{equation}
\lim_{t \rightarrow \infty} {}^i b(t) = b^*  \, \text{and} \, \lim_{t\rightarrow \infty} ({}^i b - {}^j b ) = 0, \forall i, j \in R
\label{eqn:concensus_range}
\end{equation}
where $b^* = (r^*,\theta^*)$ is the consensus over the anchor's radius and orientation among all robots. Moreover, this can be proven to happen with iterations at most the diameter of the graph, i.e., $\tau \leq \mathbb{D}(\mathcal{G})$ \cite{alonso2016distributed}.

\subsection{Anchor-oriented Coverage (AOC) Controller}
\label{sec:covergae_controller}
Finally, the anchor-oriented Voronoi partitioning by each robot in their local environments ${}^i \mathcal{Q}$ can be calculated as
\begin{align}
{}^i \mathcal{V}_i  \left({}^i{\mathbf{P}}\right)  = \{{}^iq \in {}^{i}\mathcal{Q} \mid \left\|({}^iq-{}^i{\mathbf{p}_{AP}}) - ({}^i\mathbf{p}_i-{}^i{\mathbf{p}_{AP}})\right\| \nonumber \\
\leq \left\| ({}^iq-{}^i{\mathbf{p}_{AP}}) - {}^i \mathbf{R}_j ({}^j\mathbf{p}_j-{}^j{\mathbf{p}_{AP}}) \right\| , \forall j \in \mathcal{N}_i \}
\label{eq: anchor_oriented_voronoi}
\end{align}

Here, ${}^i \mathbf{P} = \{ _i \mathbf{p}_1, _i \mathbf{p}_2, \ldots, _i \mathbf{p}_{n}\}$ is the vector containing the relative positions of all robots w.r.t. $i's$ FoR, and ${}^i \mathbf{R}_j$ is a rotation matrix to transform the coordinates from j to i. 
We posit that robots know their initial orientations (to obtain $_{i}\mathbf{R}^j$), assuming that robots start from a command station, where initial configuration can be controlled. They can also obtain this in real-time using their magnetometers in the IMU (magnetic heading acting as a proxy for global orientation). This is useful for aerial robots but not ground robots, given that a magnetic field is sensitive to man-made structures, especially in urban environments. Nevertheless, this limitation can be alleviated by the use of Angle-of-Arrival estimation techniques \cite{jadhav2022toolbox} or RSSI-based relative bearing estimation \cite{parasuraman2019consensus} for the relative bearing of neighboring robots. Precise dependency of this rotation is out of the scope of this paper.

We can now derive Lloyd's centroid-seeking controller equivalent for this anchor-oriented Voronoi partitioning as
\begin{equation}
{}^i \dot{\mathbf{p}}_i = - k\left(C_{{}^i \mathcal{V}_i }-{}^i{\mathbf{p}}_i\right)
\label{eq: positionController_relativeLocalization} 
\end{equation}
\begin{equation}
C_{{}^i \mathcal{V}_i} = \frac{1}{M_{{}^i \mathcal{V}_i}} \int_{{}^i \mathcal{V}_i} {}^i q \phi({}^i q) dq  \text{ and } M_{{}^i \mathcal{V}_i}  = \int_{{}^i \mathcal{V}_i}  \phi({}^i q) dq .
\label{eq: positionController_relativeLocalization_centroid}
\end{equation}
It is important to emphasize that within this context, each robot $i$ possesses its estimate of the Voronoi partition within its local FoR, and both ${}^i \mathcal{V}_i$ and ${}^i \mathbf{P}$ are defined in local FoR, ensuring independence from a global FoR.
Consequently, each robot computes its velocity by relying on its localized estimation of the centroid.
In this way, the proposed method enables efficient coverage in GPS-denied environments.
Algorithm~\ref{algo:aoc} provides a pseudocode representation of this methodology.

\begin{algorithm}[t]
\KwIn{
set of robots $R$, their positions in local FoR, and workspace boundary (b).\\
 }
 \KwOut{
    Coverage Partition ${}^{i}\mathcal{Q}$ of Robot $i$ in FoR($i$).
 }
Estimate workspace boundary ${}^i b$, $\forall i \in R$ \\
\SetKwBlock{RConsensus}{Consensus For Anchor Boundary}{end}
\RConsensus
{
\For{neighbor robot $j \in \mathcal{N}_i$}
{
    Send ${}^i b (t)$  \\
    Receive ${}^i b (t)$  \\
}
Apply Eq.~\eqref{eq:consensus} until boundary consensus
}


 \While{coverage not converged}{
  \For{Each robot $i \in R$} 
    {
    Find anchor position ${}^i\mathbf{p}_{AP}$ using local sensor. \\
    Get info from neighbors, ${}^j\mathbf{p}_{j}$ and ${}^j\mathbf{p}_{AP}$. \\
    Get Voronoi partition for all robots using Eq.~\eqref{eq: anchor_oriented_voronoi}. \\
    Apply position controller ${}^i\dot{p}_i$ in Eq \eqref{eq: positionController_relativeLocalization}.\\
    }
    \If{$C_{{}^i \mathcal{V}_i}- {}^i\mathbf{p}_i \leq \epsilon, \forall I \in R$}{
      convergence = True;\\
      Use the coverage partitions for successive multi-robot tasks (e.g., adaptive sampling, cooperative tracking, exploration, etc.).
   }
 }
\caption{Anchor-Oriented Multi-Robot Coverage Partitioning without Global Localization}
\label{algo:aoc}
\end{algorithm}

\subsection{Theoretical Analysis}
Here, we theoretically analyze the correctness of the proposed approach in obtaining optimal partitions.

\def\figwidth{0.41}
\begin{figure*}[ht]
    \centering
    \small
    \renewcommand{\arraystretch}{0.2}
    \begin{tabular}{ccccc}
    \hspace{-0.18in}
        \textbf{\underline{Initial}} & 
        \hspace{-0.10in}\textbf{\underline{CVT (GPS-based)}} & \hspace{-0.10in}\textbf{\underline{AOC (no-consensus)}} & \hspace{-0.18in}
        \textbf{\underline{AOC (ours)}} &\hspace{-0.18in} \textbf{\underline{Regret $r(t)$}}\\
        \hspace{-0.18in}
        \includegraphics[width=\figwidth\columnwidth]{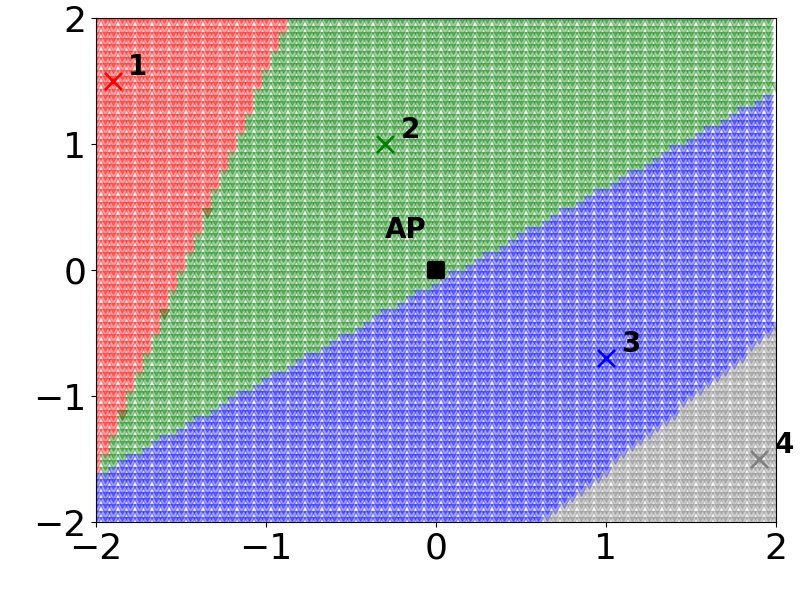} &
        \hspace{-0.18in}\includegraphics[width=\figwidth\columnwidth]{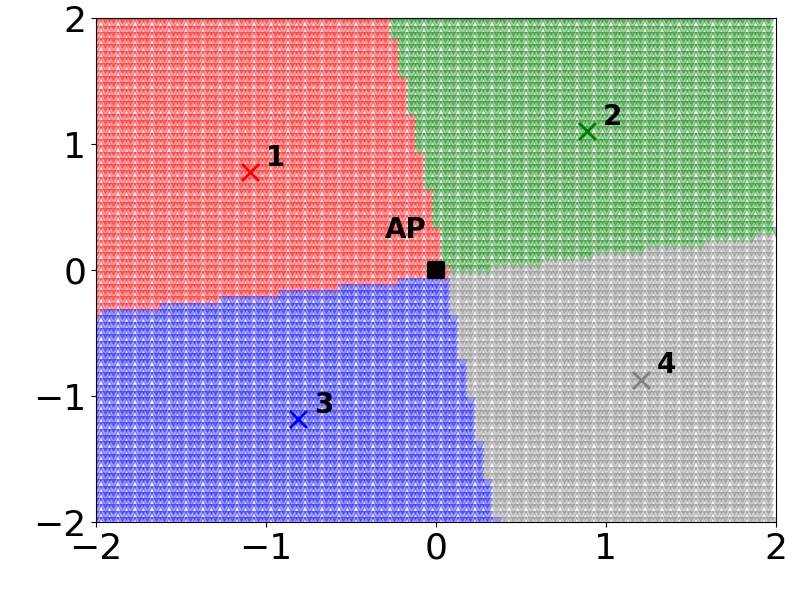}&
        \hspace{-0.18in}\includegraphics[width=\figwidth\columnwidth]{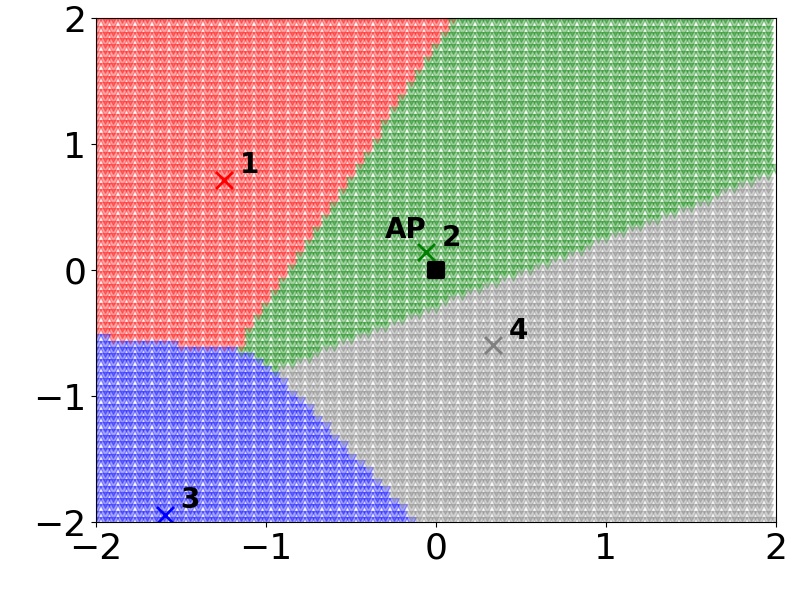} & 
        \hspace{-0.18in}\includegraphics[width=\figwidth\columnwidth]{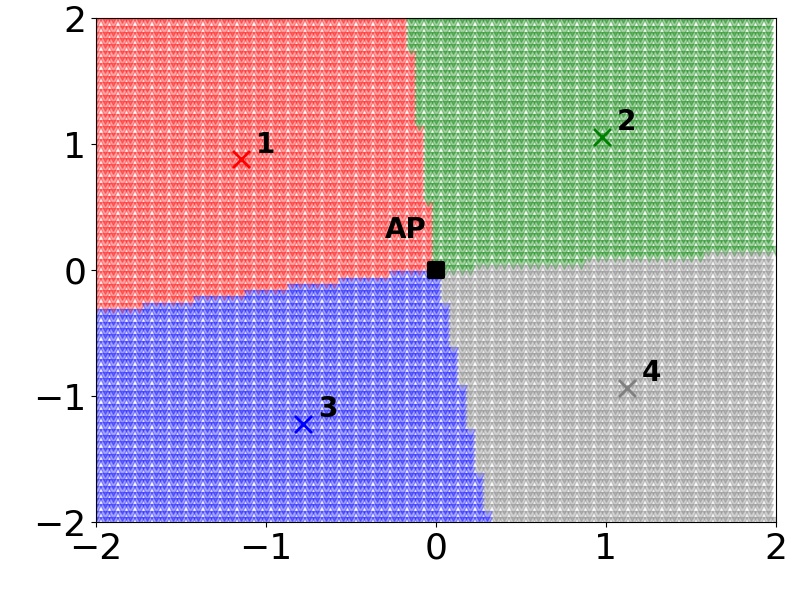} & 
        \hspace{-0.18in}\includegraphics[width=\figwidth\columnwidth]{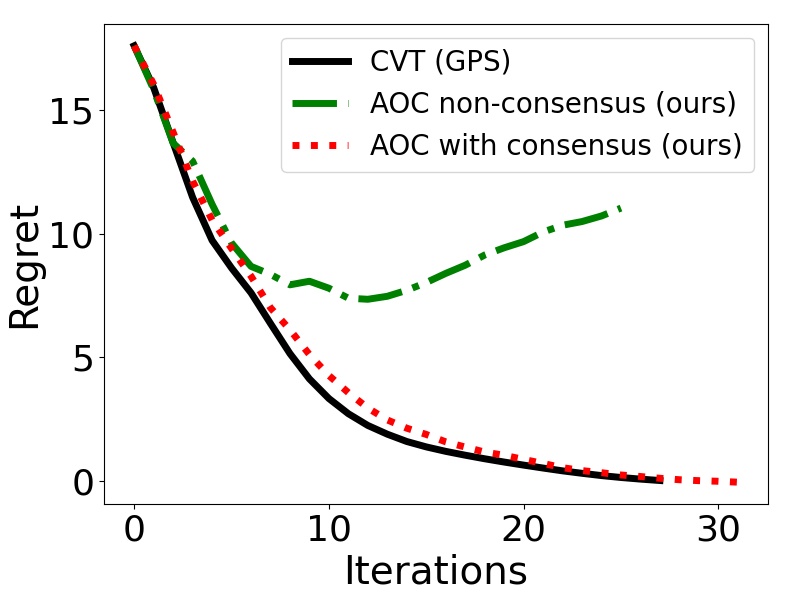}
     
    \end{tabular}
    \caption{Coverage results in a $4\times4 m^2$ environment. The plots show the initial locations of robots and the final locations after running different coverage controllers. The right-most plots show the locational cost from Eq.~\eqref{tarditionalLLoydCost}, which shows the AOC with consensus closely tracks CVT's performance.}
    \label{fig:N4_0.1std}
    
\end{figure*}
\def\figwidth{0.41}
\begin{figure*}[t]
    \centering
    \small\renewcommand{\arraystretch}{0.2}
    \begin{tabular}{ccccc}
    \hspace{-0.18in}
         \textbf{\underline{Initial}} & 
         \hspace{-0.10in}\textbf{\underline{CVT (GPS-based)}} & \hspace{-0.10in}\textbf{\underline{AOC (no-consensus)}} & \hspace{-0.2in}
         \textbf{\underline{AOC (ours)}} &\hspace{-0.2in} \textbf{\underline{Regret $r(t)$}}\\
         \hspace{-0.18in}
        \includegraphics[width=\figwidth\columnwidth]{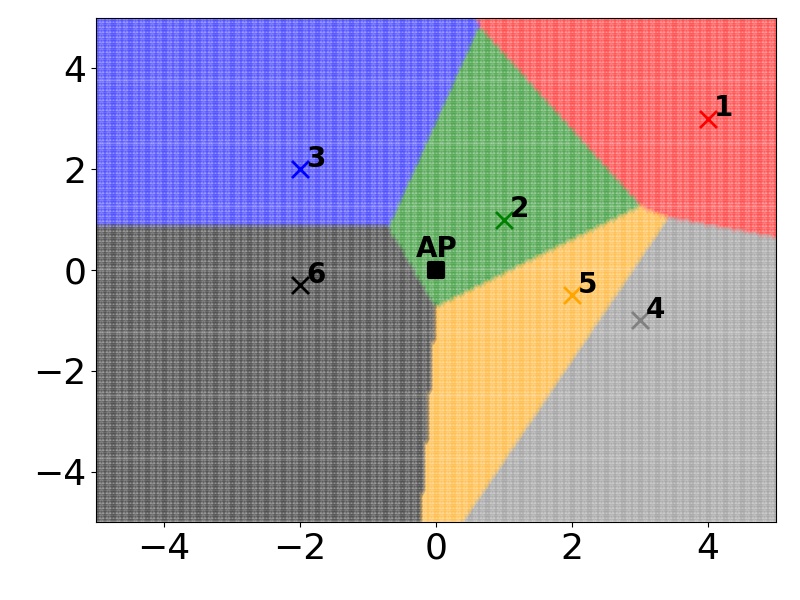} &
        \hspace{-0.18in}\includegraphics[width=\figwidth\columnwidth]{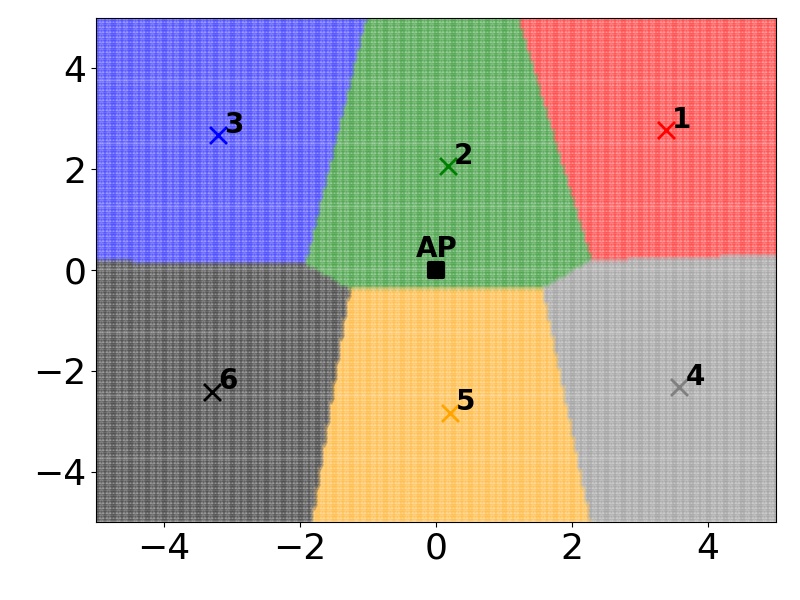}&
        \hspace{-0.18in}\includegraphics[width=\figwidth\columnwidth]{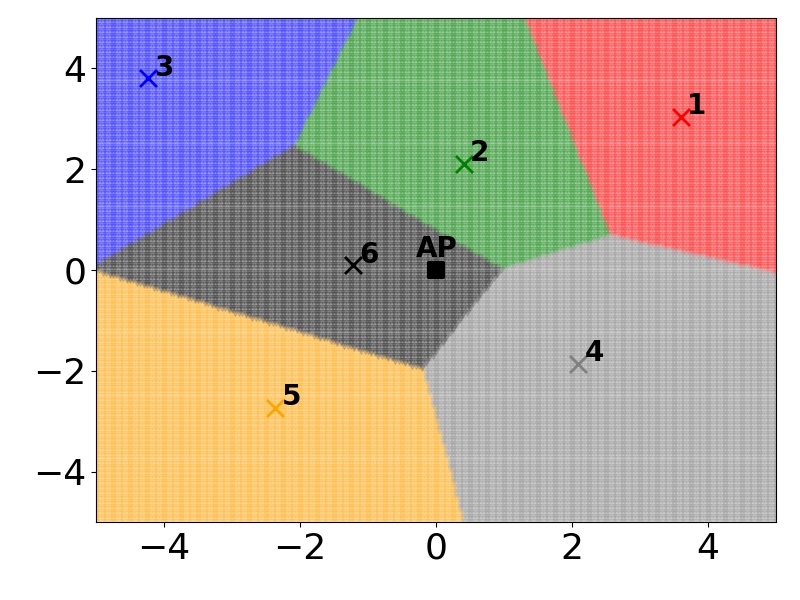} & 
        \hspace{-0.18in}\includegraphics[width=\figwidth\columnwidth]{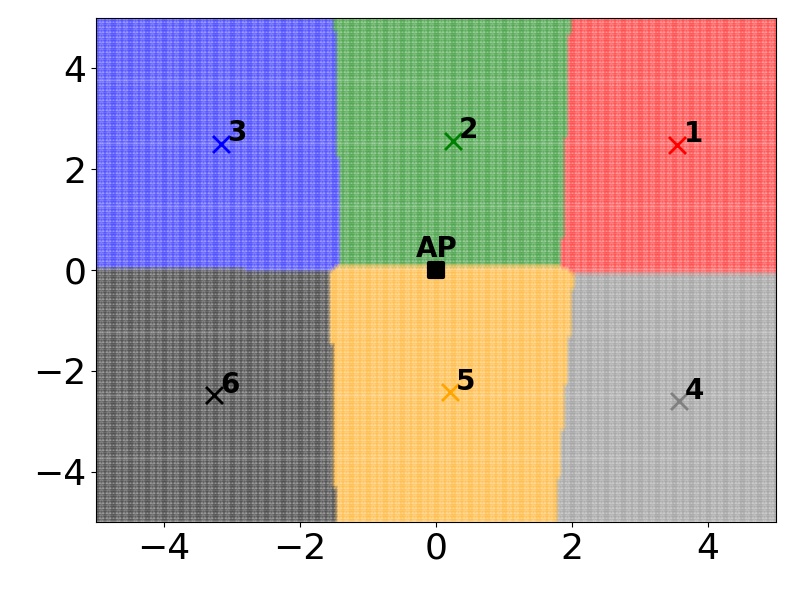} & 
        \hspace{-0.18in}\includegraphics[width=\figwidth\columnwidth]{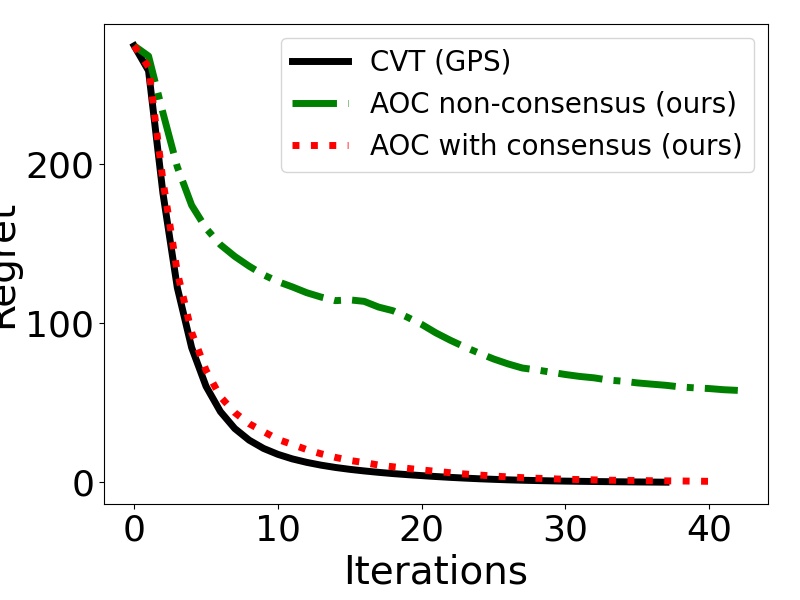}
     
    \end{tabular}
    \caption{Coverage results in a $10\times10 m^2$ environment. The coverage results and the locational cost confirm the proposed approach's capability to solve the Localized Voronoi partitioning problem with equivalent performance to global Voronoi partitioning.}

    \label{fig:N6_0.1std}
\end{figure*}

\def\figwidth{0.41}
\begin{figure*}[t]
    \centering
    \small
    \renewcommand{\arraystretch}{0.2}
    \begin{tabular}{ccccc}
     \hspace{-0.18in}
  \textbf{\underline{Initial  (Global FoR)}} & 
          \hspace{-0.08in}\textbf{\underline{Robot1 FoR}} & \hspace{-0.05in}\textbf{\underline{Robot2 FoR}} & \hspace{-0.05in}
         \textbf{\underline{Robot3 FoR}} &\hspace{-0.05in}\textbf{\underline{Robot4 FoR}}\\
        \hspace{-0.18in}        \includegraphics[width=\figwidth\columnwidth]{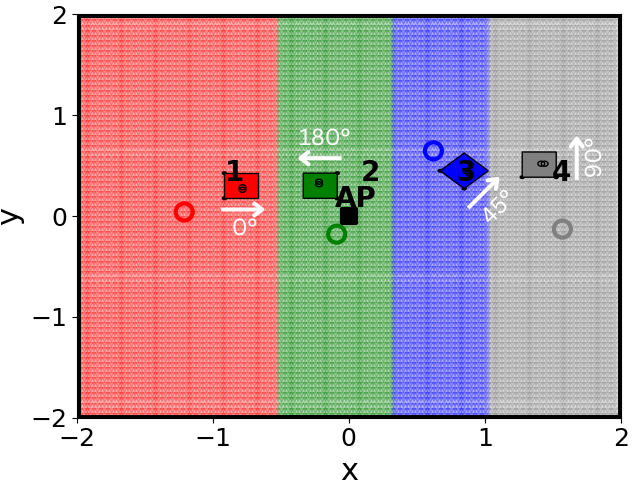} &
        \hspace{-0.18in}\includegraphics[width=\figwidth\columnwidth]{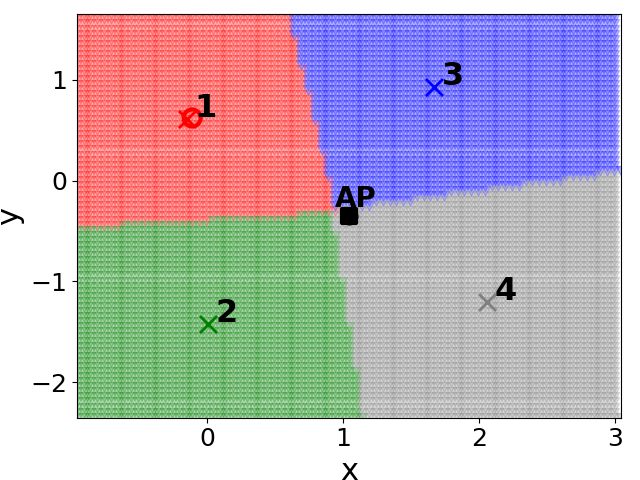}&
        \hspace{-0.18in}\includegraphics[width=\figwidth\columnwidth]{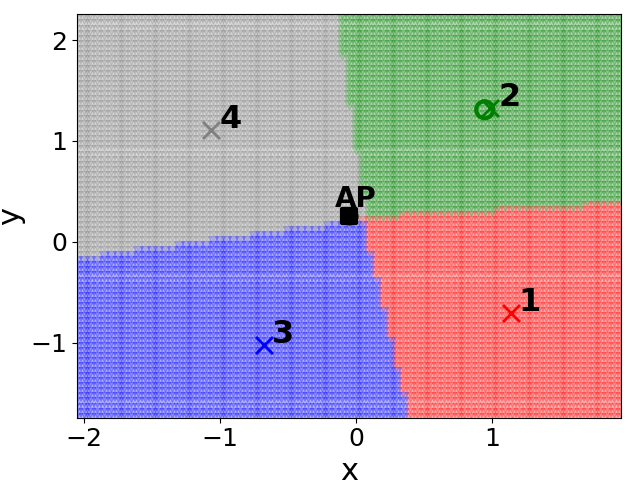} & 
        \hspace{-0.18in}\includegraphics[width=\figwidth\columnwidth]{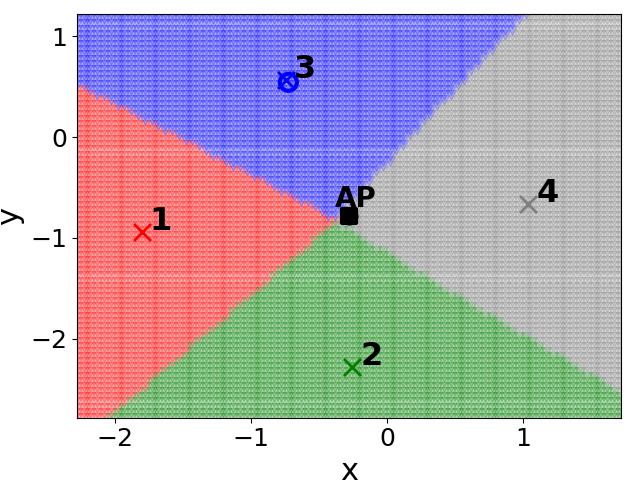} & 
        \hspace{-0.18in}\includegraphics[width=\figwidth\columnwidth]{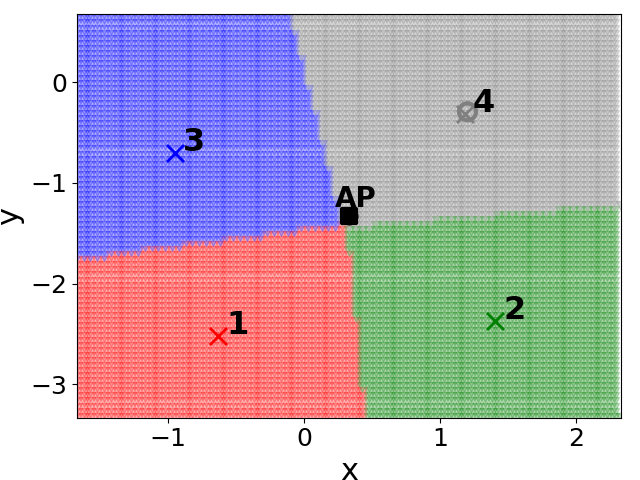}
    \end{tabular}
    \caption{ Coverage results when four robots have different initial orientations 0°,90°,45°, and 180°, respectively. The global FoR aligns with the first robot.}   \label{fig:N4_orientation}
\end{figure*}

\begin{theorem}
Applying anchor range consensus in \eqref{eq:consensus}, the anchor-oriented Voronoi partitioning in \eqref{eq: anchor_oriented_voronoi} with a local bounded region ${}^i \mathcal{Q} \subset \mathbb{R}^2$ will be equivalent to the global Voronoi partitioning in \eqref{eq: Voronoi} with a global bounded region $\mathcal{Q}$.
\label{theorem:partioning}
\end{theorem}
\begin{proof}
To prove the equivalence of Voronoi partitioning, we first re-arrange the condition in Eq.~\eqref{eq: anchor_oriented_voronoi} as
\begin{equation}
\left\|{}^iq-{}^i{\mathbf{p}_{i}}\right\| 
\leq \left\| ({}^iq-{}^i{\mathbf{p}_{AP}}) - {}^i \mathbf{R}_j ({}^j\mathbf{p}_j-{}^j{\mathbf{p}_{AP}}) \right\|
\end{equation}
After manipulation with ${}^i \mathbf{R}_j$, we see that this is equivalent to a localized Voronoi with relative localization:
\begin{equation}
\|{}^i q_i - {}^i{}\mathbf{p}_i\| \leq \|{}^i q_i - {}^i{}\mathbf{p}_j\|, \quad \forall {}^i q \in {}^i \mathcal{Q} \label{local}
\end{equation}
And the global Voronoi condition, where g is the global FoR (for GPS-based coverage):
\begin{equation}
\|{}^g{}q - {}^g{}\mathbf{p}^i\| \leq \|{}^g{}q - {}^g{}\mathbf{p}^j\|, \quad \forall {}^g{}q \in {}^g \mathcal{Q} \label{global}
\end{equation}
We posit the existence of a consensus on the AP boundary $b^* \forall i \in R$ using the distributed consensus in Eq.~\eqref{eq:consensus}. Since all robots center their partitioning around the same anchor Eq.~\eqref{eq:square}, ${}^i \mathcal{Q} \approx {}^j{}\mathcal{Q} \forall (i,j) \in E(\mathcal{G})$. 
This results in ${}^i{}\mathcal{Q} \approx {}^g{}\mathcal{Q}$, where ${}^g{}\mathcal{Q}$ is the global workspace square region centered at AP.
Hence, based on the given equivalence above:
\begin{equation}
T(\mathbf{{}^i{}\mathbf{p}_i}) = \mathbf{{}^g{}\mathbf{p}_i} \label{transformation}
\end{equation}

By applying transformation $T$ (from Eq.~\eqref{transformation}:
\begin{align*}
\| T(\mathbf{{}^i{}q_i}) - T(\mathbf{{}^i{}\mathbf{p}_i})\| &\leq \| T(\mathbf{{}^i{}q_i}) - T(\mathbf{{}^i{}\mathbf{p}_j})\| \\
\|{}^g{}q - {}^g{}\mathbf{p}_i\| &\leq \|{}^g{}q - {}^g{}\mathbf{p}_j\|
\end{align*}

This is consistent with Eq.~\eqref{global}.
Applying the inverse transformation $T^{-1}$ to Eq.~\eqref{global}:
\begin{align*}
\| T^{-1}(\mathbf{{}^g{}q}) - T^{-1}(\mathbf{{}^g{}\mathbf{p}_i})\| &\leq \| T^{-1}(\mathbf{{}^g{}q}) - T^{-1}(\mathbf{{}^g{}\mathbf{p}_j})\| \\
\|\mathbf{{}^i{}q_i} - \mathbf{{}^i{}\mathbf{p}_i}\| &\leq \|\mathbf{{}^i{}q_i} - \mathbf{{}^i{}\mathbf{p}_j}\| 
\end{align*}
This is consistent with Eq.~\eqref{local}.

Thus, the localized Voronoi condition \eqref{eq: anchor_oriented_voronoi} under the transformation $T$ is equivalent to the global Voronoi condition \eqref{global}, and vice versa under inverse transformation  $T^{-1}$.
\end{proof}

\textit{Corollary 1} The centroid-seeking position controller in Eq.~\eqref{eq: positionController_relativeLocalization} will minimize the coverage cost in Eq.~\eqref{tarditionalLLoydCost}, which for both Global Voronoi partitioning and Localized Voronoi partitioning will be equivalent. i.e.,
\begin{align}
H_{{}^i \mathcal{V}}({}^i \mathbf{P}) &= \sum_{i=1}^n \int_{\mathcal{V}_i} \frac{1}{2} ||{}^i q - {}^i \mathbf{p}_i||^2 \phi({}^i q) dq , \label{eq:localLloydcost}
\\
H_{{}^i \mathcal{V}}({}^i \mathbf{P}) &\approx H_{\mathcal{V}}(\mathbf{p}_1, ..., \mathbf{p}_n) \nonumber
\end{align}
\begin{proof}
Theorem~\ref{theorem:partioning} has proven the equivalence of the local and global partitioning; consequently, the Lloyd controller will converge to the global cost as proved in \cite{cortes2004coverage}.
\end{proof}

\textit{Corollary 2} The distance between the centroids (or final positions) of the Localized Voronoi partitioning of the robots in their local FoRs will be the same as that of other robot's FoRs,  as long as consensus on $r^*$ is achieved. i.e., 
\begin{equation}
 |{}^i c_j - {}^i c_i| = |{}^j c_j - {}^j c_j|  \quad \forall i,j \in R
 \label{eqn:Voronoi}
 \end{equation}
Here, ${}^{i}c_{j}$ refers to the centroid of $j$ in FoR of $i$ in Eq.~\eqref{eq: positionController_relativeLocalization_centroid}.
\begin{proof}
Trivial. Based on the proofs of Theorem~\ref{theorem:partioning} and Corollary 1, it is not difficult for one to see that the robots reach the convergence to their centroids as if a global square region centered around the AP with a width $2r^*$ existed.
\end{proof}

\textit{Remark 1} The anchor can be dynamic (since the robots estimate the anchor's location and share with others at every time instance, $k$, converging quickly to $r^*$). Also, a minimum of one physical anchor can be sufficient to determine the centroids for coverage; however, in the case of multiple dynamic anchors, the moving centroid of centroids of anchor ranges can be considered for coverage (see \cite{nascimento1999algorithm} for details).

\textit{Remark 2} The anchor-oriented Voronoi partitioning can be made robust for noisy sensor measurements, for instance, by combining with the Guaranteed Voronoi approaches \cite{chevet2019guaranteed}.

\subsection{Regret Analysis}
Let's define a coverage regret metric as the discrepancy between the current locational cost and the optimum achievable cost for a given set of robots.
\begin{equation}
    r(t) = H_{\mathcal{V}}(\mathbf{p}_1(t), ..., \mathbf{p}_n(t)) - H_{\mathcal{V}}(\mathbf{p}_1^*, ..., \mathbf{p}_n^*) 
    \label{eqn: regret}
\end{equation}
The cumulative regret is $R(t) = \sum_i^t r(t)$.

\begin{theorem}
Let $T$ be the time horizon, and assume each robot in a multi-robot system estimates the position of a common anchor with uncertainty modeled as $N(^i \mu_{AP}, ^i \sigma_{AP}^2)$ for each robot $i$. For a given number of robots and the environment size defined by $b^*$, the cumulative regret $R(T)$ after $T$ steps, due to suboptimal coverage decisions influenced by the anchor position estimation, satisfies:
\begin{equation}
    R(T) \leq \mathcal{O}\left(\sqrt{T(1 + \sigma_{AP}^2)}\right),
\end{equation}
where $\sigma_{AP}^2$ denotes the variance in the Gaussian noise affecting the anchor position estimation.
\label{theorem:regret}
\end{theorem}

\begin{proof}
To establish the upper bound on cumulative regret due to noise in anchor position estimation, we begin by expressing the impact of Gaussian noise on the coverage cost function. Let $\mathbf{p}_i(t)$ represent the position of robot $i$ at time $t$, and $\mathbf{p}_i^*$ denote the optimal position of robot $i$. Given the Gaussian noise in anchor position estimation, the observed position $\tilde{\mathbf{p}}_i(t)$ deviates from the true position $\mathbf{p}_i(t)$ as:
\begin{equation}
    \tilde{\mathbf{p}}_i(t) = \mathbf{p}_i(t) + \epsilon_i(t), \quad \epsilon_i(t) \sim \mathcal{N}(0, \sigma_{AP}^2),
\end{equation}
where $\epsilon_i(t)$ is the noise term for robot $i$ at time $t$.

The effect of this noise on the locational cost function $H_{\mathcal{V}}$ can be analyzed through its impact on the decision-making process of the robots. Specifically, the discrepancy in the observed positions leads to suboptimal coverage decisions, thereby increasing the coverage regret.

Considering the definition of cumulative regret, the expectation of regret at each time step due to the noise can be modeled as a function of the variance $\sigma_{AP}^2$:
\begin{equation}
    \mathbb{E}[r(t)] = \mathbb{E}\left[H_{\mathcal{V}}(\tilde{\mathbf{p}}_1(t), ..., \tilde{\mathbf{p}}_n(t)) - H_{\mathcal{V}}(\mathbf{p}_1^*, ..., \mathbf{p}_n^*)\right],
\end{equation}
which is directly influenced by the distribution of $\epsilon_i(t) \forall i$.
To further elucidate the relationship between $\mathbb{E}[r(t)]$ and the distribution of $\epsilon_i(t)$, we expand the expectation as follows:
\begin{multline}
    \mathbb{E}[r(t)] = \int_{-\infty}^{\infty} \Bigr[H_{\mathcal{V}}(\mathbf{p}_1(t) + \epsilon_1(t), ..., \mathbf{p}_n(t) + \epsilon_n(t)) - \\
    H_{\mathcal{V}}(\mathbf{p}_1^*, ..., \mathbf{p}_n^*)\Bigr] \prod_{i=1}^{n}f(\epsilon_i(t))d\epsilon_i(t),
\end{multline}
where $f(\epsilon_i(t))$ is the probability density function of the Gaussian distribution for the noise affecting robot $i$'s position estimate at time $t$. 

The impact of $\epsilon_i(t)$ on the locational cost function can be seen as introducing variability to the actual positions $\mathbf{p}_i(t)$, leading to suboptimal decisions. This variability can be quantified in terms of the variance $\sigma_{AP}^2$ of the noise distribution.
To transition from the expected single-step regret to the cumulative regret over all time steps $T$, we sum the expectations of $r(t)$ across all $t$ from $1$ to $T$, acknowledging the additive nature of regret over time:
\begin{equation}
    R(T) = \sum_{t=1}^{T} \mathbb{E}[r(t)].
\end{equation}
Considering the compounding effect of Gaussian noise on the decision-making process across all robots and time steps, we leverage the Central Limit Theorem (CLT) for summing independent random variables (the noise terms $\epsilon_i(t)$) to argue that the sum of variances over $T$ steps contributes to the scaling of the cumulative regret. Specifically, the cumulative effect of noise on regret is proportional to the root of $T$ times the average effect of noise at a single step, leading to:
\begin{equation}
    R(T) \leq \mathcal{O}\left(\sqrt{T(1 + \sigma_{AP}^2)}\right),
\end{equation}
where the root of $T$ factor arises from the CLT and the $(1 + \sigma_{AP}^2)$ term represents the compounded impact of noise variance on the cumulative regret, reflecting both the direct effect of noise and its variance on suboptimal decisions over time.
This bound signifies that as the uncertainty $\sigma_{AP}^2$ in anchor position estimation increases, the achievable cumulative regret grows, reflecting the inherent challenges in optimizing coverage under uncertain conditions.
\end{proof}

\textbf{Algorithmic Complexity:} The proposed approach is distributed, and each robot performs similar operations and, therefore, is highly scalable for the number of robots in the system. Per Alg.~\ref{algo:aoc}, each robot shares information with other robots through a stable communication channel with the same payload containing the current local position, AP position prediction, and locally determined bounded region; hence, the communication cost will be $C(\tau n)$, for $\tau$ number of rounds for consensus in $n$-number of robots. The approach begins with localized Voronoi partitioning with a computational complexity $O(n^2)$. Each robot then determines the boundary based on the consensus approach discussed in Sec.~\ref{sec:anchor_boundary_consensus}, which requires $O(\tau n)$, where $\tau \leq \mathbb{D}(\mathcal{G})$. Finally, robots perform coverage using AOC controller discussed in Sec.~\ref{sec:covergae_controller} in $O(wknd)$, for $w$ iterations, $k$ centers, $n$-number of robots and $d$ dimensions. Here, we consider $k=n$ for the same number of centers and robots and 2-dimensional space; hence, the required time for coverage will be $O(2wn^2)$. Each of these operations performs sequentially, so the overall time complexity for coverage will be $O(n^2+\tau n + 2wn^2) \rightarrow O(wn^2)$ as $\tau \leq n$. $w$ depends on the size of the environment and the robots' initial positions. For our analysis, we set it to a constant so the simplified time complexity will be $O(n^2)$.


\def\figwidth{0.37}
\begin{figure*}[ht]
    \centering
    \small
    \renewcommand{\arraystretch}{0.2}
    \begin{tabular}{ccccc} 
    \hspace{-0.12in}
        \textbf{\underline{Initial}} & \hspace{-0.1in}\textbf{\underline{Timestep 20}} & 
        \hspace{-0.1in} \textbf{\underline{Timestep 40}} & \hspace{-0.1in}\textbf{\underline{Timestep 60}} & 
        \hspace{-0.1in} \textbf{\underline{Timestep 90}} \\
        \hspace{-0.12   in}
        \includegraphics[width=\figwidth\columnwidth]{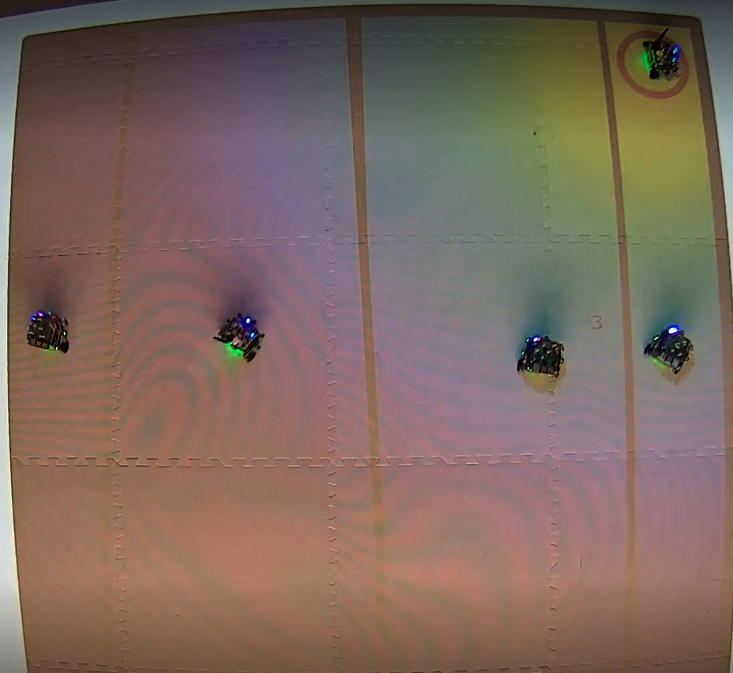} &
        \hspace{-0.05in}
        \includegraphics[width=\figwidth\columnwidth]{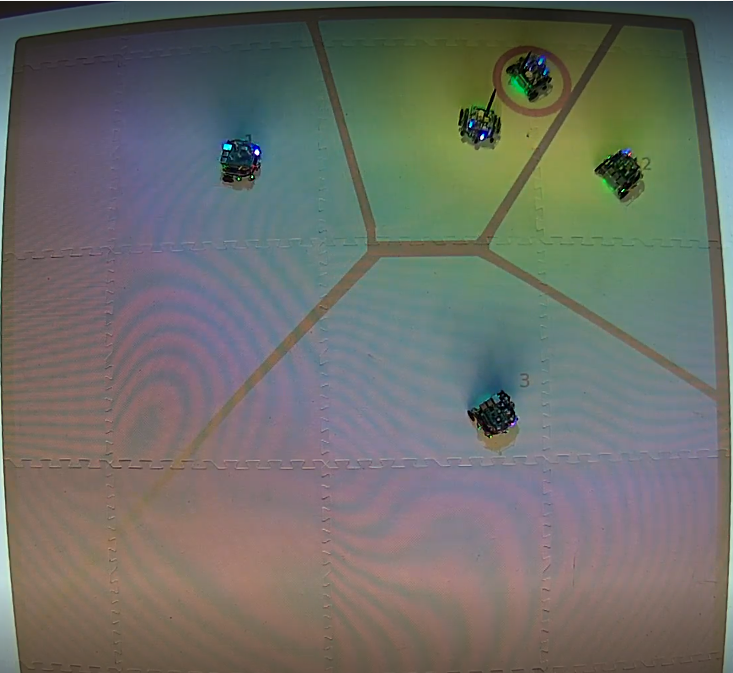}&
        \hspace{-0.05in}
        \includegraphics[width=\figwidth\columnwidth]{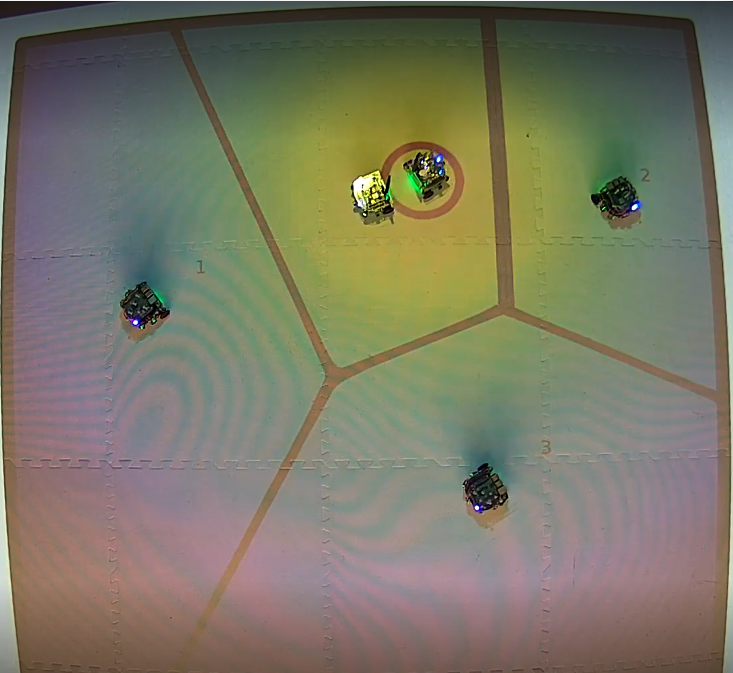} &
         \hspace{-0.05in}
        \includegraphics[width=\figwidth\columnwidth]{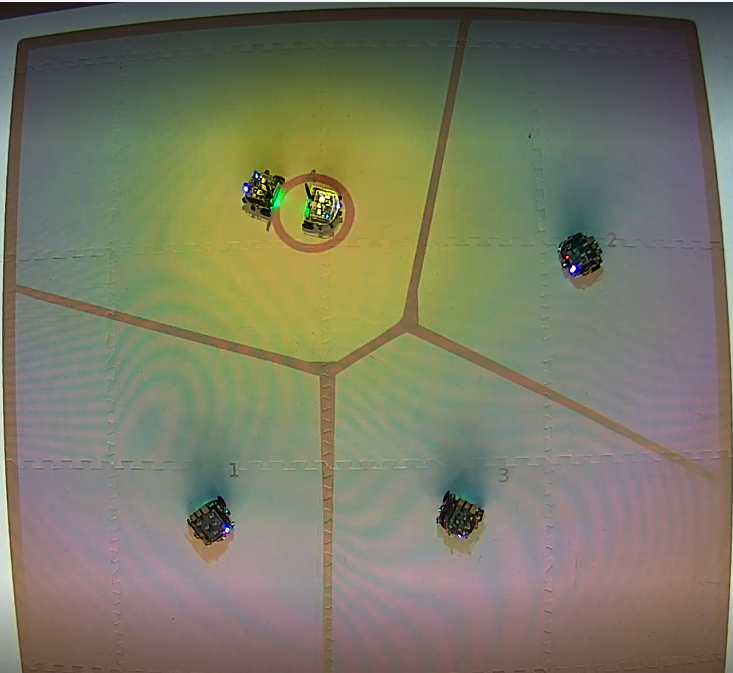}&
        \hspace{-0.05in}
        \includegraphics[width=\figwidth\columnwidth]{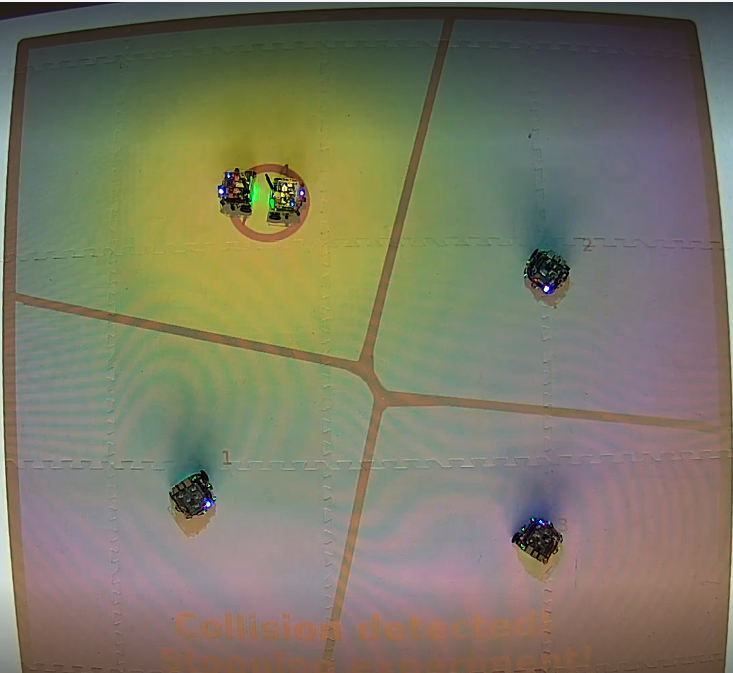}

    \end{tabular}
    \caption{Dynamic environmental processes use-case demonstration in Robotarium using the proposed AOC approach. The plots depict the density source (red-circled robot) movement from left to right. The source robot creates a density distribution around it, which four coverage-tasked robots would measure and use as $\phi(q)$. Note - the anchor position (AP) is at the center (static) of the environment, so the localized workspace bounds do not change.}
    \label{fig:N4_targetTracking}
\end{figure*}

\def\figwidth{0.37}
\begin{figure*}[ht]
    \centering
    \small
    \renewcommand{\arraystretch}{0.2}
    \begin{tabular}{ccccc} 
    \hspace{-0.12in}
        \textbf{\underline{Initial}} & \hspace{-0.1in}\textbf{\underline{Timestep 30}} & 
        \hspace{-0.1in} \textbf{\underline{Timestep 60}} & \hspace{-0.1in}\textbf{\underline{Timestep 90}} & 
        \hspace{-0.1in} \textbf{\underline{Timestep 120}} \\
        \hspace{-0.12in}
        \includegraphics[width=\figwidth\columnwidth]{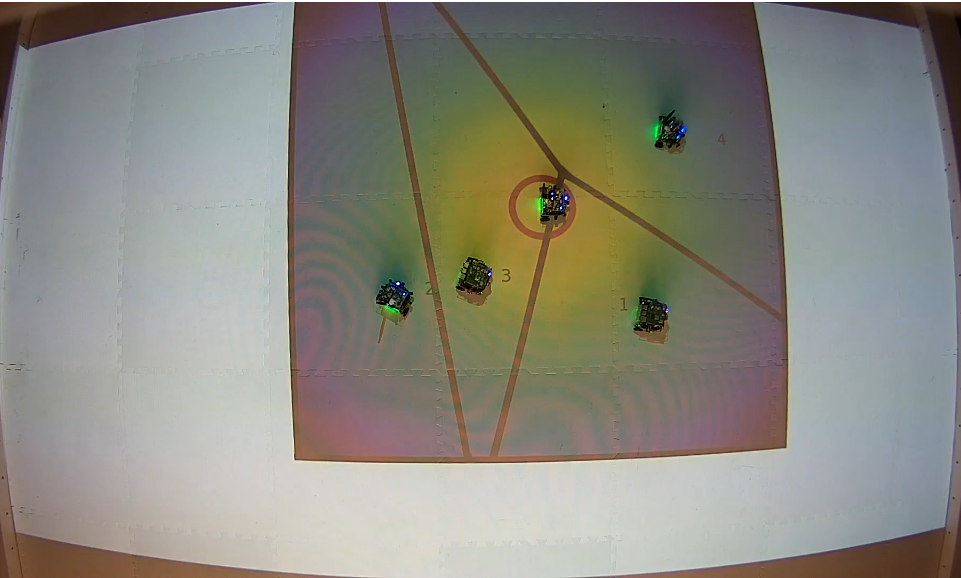} &
        \hspace{-0.05in}
        \includegraphics[width=\figwidth\columnwidth]{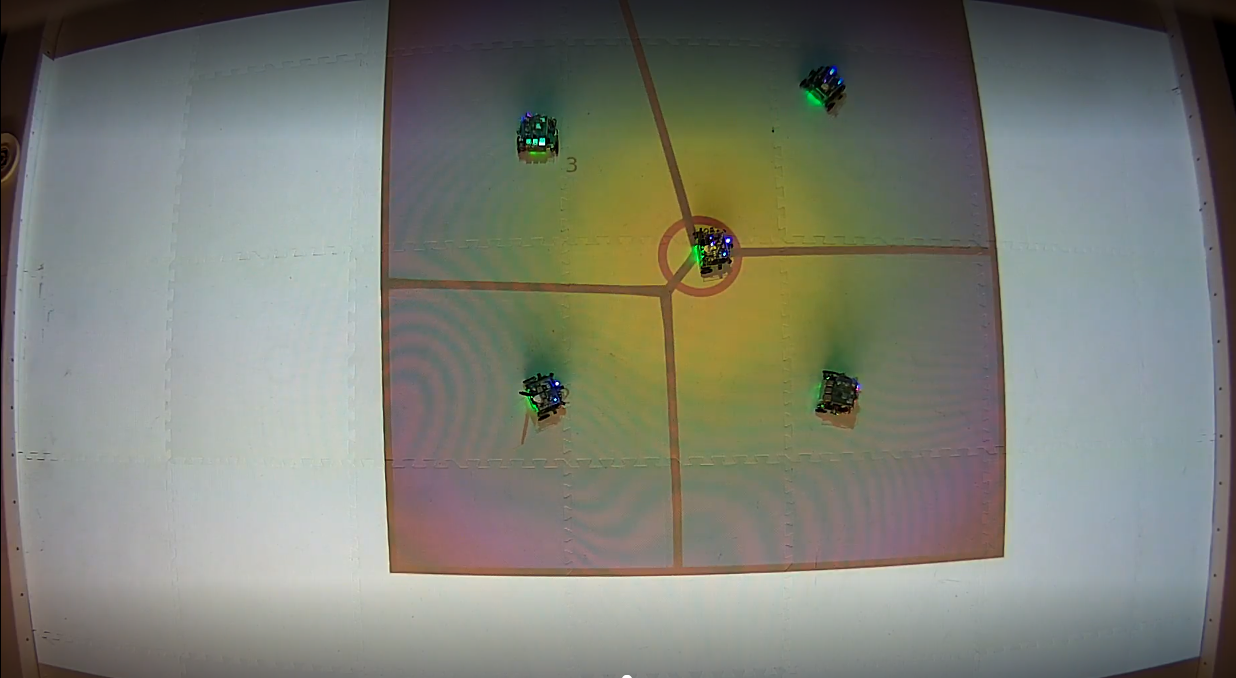}&
        \hspace{-0.05in}
        \includegraphics[width=\figwidth\columnwidth]{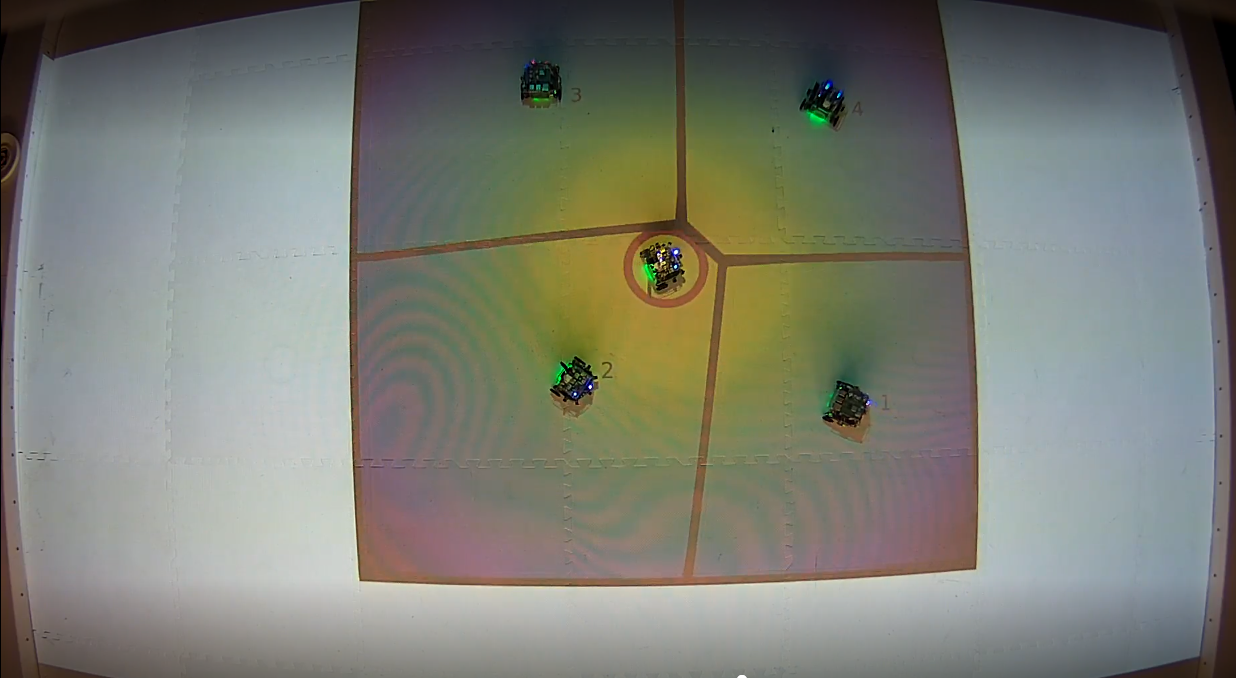} &
         \hspace{-0.05in}
        \includegraphics[width=\figwidth\columnwidth]{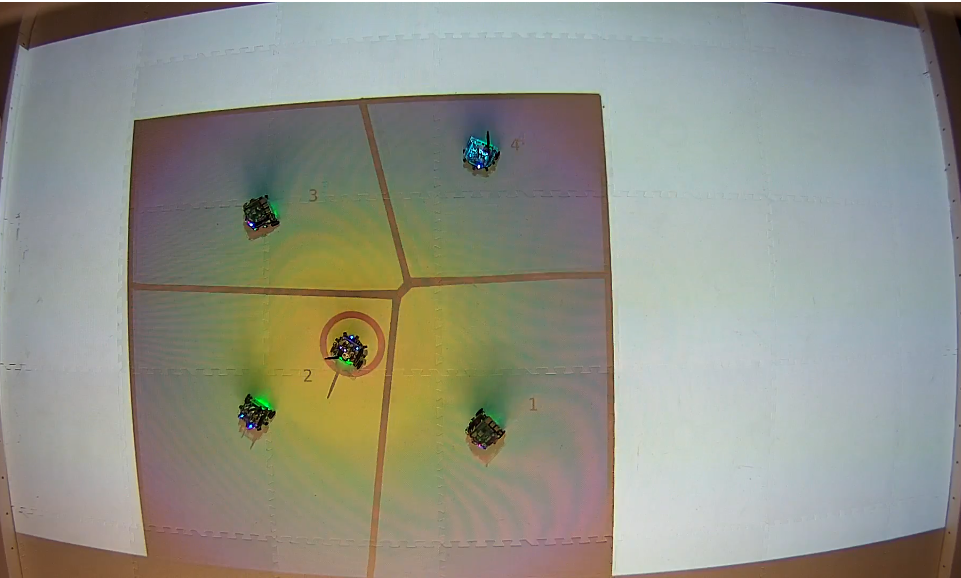}&
        \hspace{-0.05in}
        \includegraphics[width=\figwidth\columnwidth]{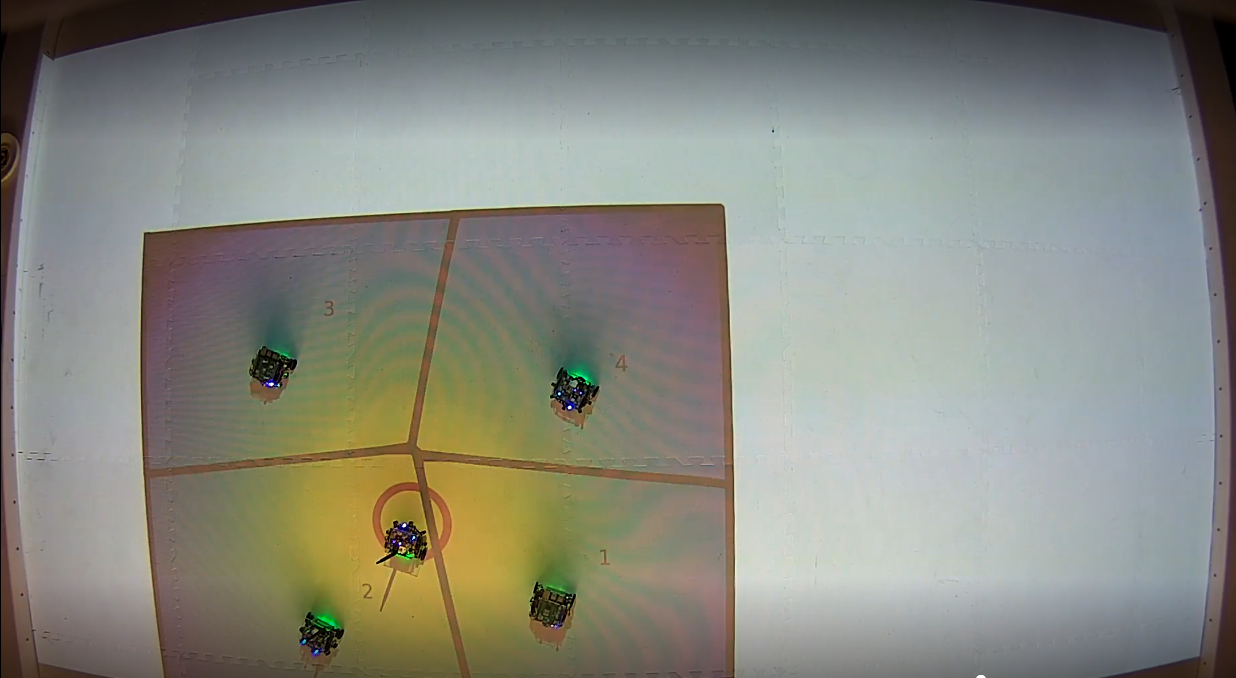}

    \end{tabular}
    \caption{Moving anchor (target tracking) demonstration in Robotarium using the proposed AOC approach. The anchor (red-circled robot) moves from the top-right to the bottom-left location. over time. Since the Localized Voronoi partitioning is centered around the anchor, all four coverage-tasked robots move (adapt) their coverage workspace along with the anchor and keep covering the workspace optimally.}
    \label{fig:N4_dynamicAP}
\end{figure*}

\section{Experiments and Results}
\label{sec:experiments}
We performed experiments using Robotarium \cite{wilson2021robotarium} python simulator, utilizing two separate environment sizes: a small-scale $4\times 4 m^2$ with 4 robots and a large-scale $10\times 10 m^2$ with 6 robots. Our evaluation involved testing the proposed \textbf{anchor-oriented coverage (AOC)} approach under various scenarios and application use cases, as described below. 
Samples of the simulation trials and real-world demonstrations are presented in the \textbf{attached video}.


To validate our approach, we implemented a baseline centroidal Voronoi Tessellation (\textbf{CVT}) \cite{kemna2017multi}, which performs coverage controller (dynamic global Voronoi partitioning) using global localization (GPS-based).
We also implement a variant of our approach without performing consensus (\textbf{AOC no-consensus})
to show the impact of the consensus controller in Sec.~\ref{sec:anchor_boundary_consensus} in the coverage performance. 
For AOC, we simulate a Gaussian noise ${}^i p_{AP} =\mathbb{N}(p_C,0.1m)$ for the anchor position estimation, with the AP located at the center $p_C=(0,0)$ of the simulation workspace.
Furthermore, we introduced the initial disagreement of the anchor radius ${}^i r$ on the robots (simulating inconsistencies in workspace boundaries) using a Gaussian noise ${}^i r = \mathbb{N}(r^*,\sigma^2_r)$. Because of this noise, in AOC (no-consensus), the robots would have different boundary sizes around the anchor and perform Localized Voronoi partitioning with ${}^i r$, instead of $r^*$.
Parameters are set as $\sigma^2_r=1, r^*=1m$ for the small environment and $\sigma^2_r=2, r^*=2m$ for the large environment setting. CVT will use $r^*$ for the global workspace. 

Results from a sample trial of these experiments are presented in Fig. \ref{fig:N4_0.1std} for the small-scale environment and Fig.~\ref{fig:N6_0.1std} for the large-scale environment. The instantaneous regret function $r(t)$ in Eq.~\eqref{eqn: regret} calculated in a global frame is used as a performance indicator and averaged over 10 trials in each experiment. 
Notably, the results of the final positions of the robot and the locational cost reveal that the AOC exhibited performance comparable to that of GPS-based coverage, supporting Theorem~\ref{theorem:partioning}.
However, without consensus, the robots' anchor range differences would not optimally cover and minimize the cost, as expected. 
This explicates the importance of achieving consensus on the localized workspace to work with no prior agreements on the workspace boundary.
In Fig.~\ref{fig:N4_orientation}, we show the results when the robots start with different orientations. We set the orientation consensus $\theta^*=0$, so it is easier to visualize the resulting partitions from each robot's local FoRs without any skew.

\begin{figure}[t]
    \small
    \centering
    \begin{tabular}{cc}
        \textbf{\underline{Cost }} & \textbf{\underline{Cumulative regret $R$}} \\
        \vspace{-2.5mm}
        \hspace{-0.1in}\includegraphics[width=0.49\linewidth]{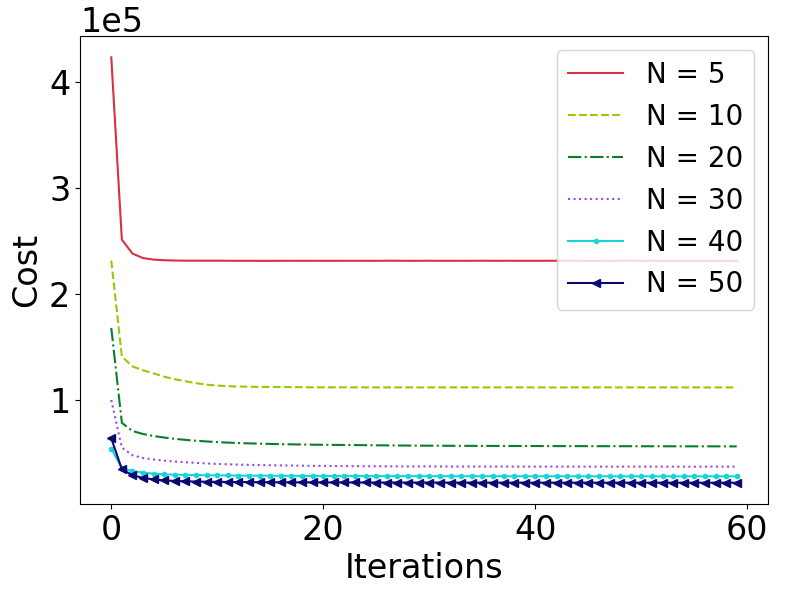} &
        \hspace{-0.1in}\includegraphics[width=0.49\linewidth]{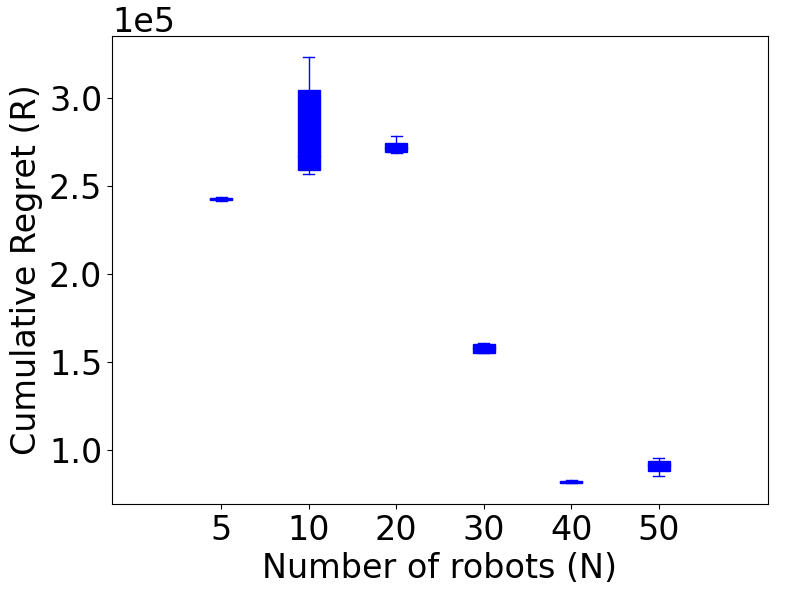} \\ 
    \end{tabular}
    \caption{Time-evolution of the cost and cumulative regret $R$ for different number of robots $N$ at the end of the experiment in a $50 \times 50m^2$ workspace.}
    \label{fig:scalability_exp}
    \vspace{-4mm}
\end{figure}

\begin{figure}[t]
    \small
    \centering
    \begin{tabular}{cc}
        \textbf{\underline{Cost }} & \textbf{\underline{Cumulative regret $R$}} \\
        \vspace{-2.5mm}
        \hspace{-0.1in}\includegraphics[width=0.49\linewidth]{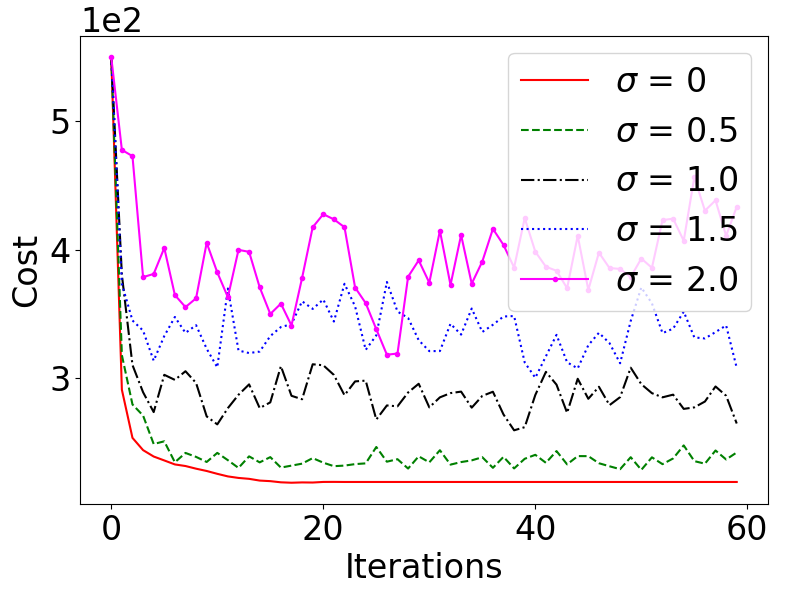} &
        \hspace{-0.1in}\includegraphics[width=0.49\linewidth]{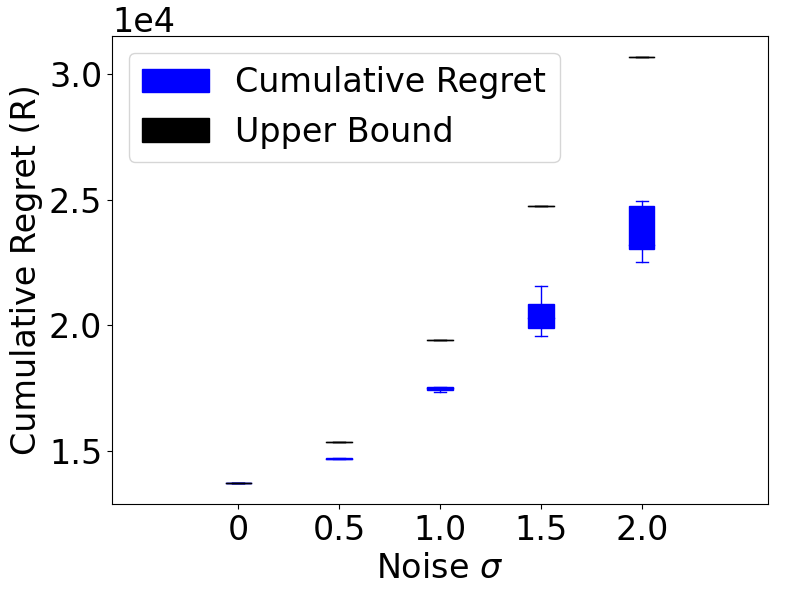} \\ 
    \end{tabular}
    \caption{Time-evolution of the cost and the cumulative regret for different noise levels with $N=10$ and $T=60$ in a $10 \times 10m^2$ workspace.}
    \label{fig:noise_exp}
    \vspace{-4mm}
\end{figure}

\subsection{Scalability Analysis}
We investigate the scalability aspect of AOC by gradually scaling up the number of robots (ranging from 5 to 50) in a $50\times 50 m^2$ (Fig. \ref{fig:scalability_exp}). We can observe that increasing the number of robots correlates with reducing both cost values and cumulative regret. This indicates that the anchor-oriented coverage is flexible enough to adapt to different numbers of robots, showcasing its scalability in a large environment. In the case of sparse graph connectivity, methods such as \cite{siligardi2019}
can be utilized to ensure connectivity maintenance. The distributed nature of the proposed algorithm ensures its scalability and functionality, contingent upon each robot maintaining connectivity with its Delaunay neighbors (the neighbors sharing an edge of their Voronoi partition). 

\subsection{Impact of the AP Prediction Uncertainty}
To assess the impact of noise level on the proposed approach, we performed experiments with 10 robots in a $10\times 10 m^2$ with noise level ($\sigma^2_{AP}$) in the AP position prediction ranging from 0\% \text{ to } 20\% of the environment size. The cumulative regret data for this experiment is presented in Fig. \ref{fig:noise_exp}. A clear trend can be observed as the cost value increases with the rise in the noise level of the AP prediction, staying within the upper bounds as expected in Theorem~\ref{theorem:regret}.

\section{Real-World Demonstration in Robotarium}
To demonstrate the applicability of our approach in real-world applications, we have conducted experiments on real robots using the Robotarium testbed \cite{wilson2021robotarium}.

\textbf{Application to Dynamic Environmental Processes}
AOC can be applied to a density-tracking scenario, where an environmental feature changes over time.  In this case, we have considered that the source of a physical process $\phi(q)$ (e.g., covering dynamic wildfires) has a probability distribution that can be measured across the environment. The goal for the pursuing robots here is to adjust their positions based on the source's current position. For the experiment setup, we have considered a $2m \times 2m$ workspace on the robotarium testbed. Initially, the dynamic feature (red-circled robot) is in the top-left corner of the workspace and is moving with a fixed velocity.
The results for this case are presented in Fig. \ref{fig:N4_targetTracking}. It can be seen from the movement graphs that robots were able to track and follow the distribution accurately.

\textbf{Dynamic Anchor}
We simulated the anchor's movement and analyzed its influence on the AOC's performance. This is analogous to a use case where the anchor is an intruder in the given environment, and the coverage-tasked robots should keep the coverage around the intruder (e.g., cooperative pursuit \cite{yang2022game} or target tracking in surveillance missions).
%
The robots aim to cover the specified boundary around the AP. Here, we have considered a square environment of width $1.5m$. The result for this scenario is shown in Fig. \ref{fig:N4_dynamicAP}. It is evident from the results that the robots could adapt dynamically to these environmental changes, highlighting the versatility of our approach with dynamic anchors (see Remark 1).

\section{Conclusion}
We proposed a novel anchor-oriented multi-robot coverage method for creating dynamic Voronoi partitions in GPS-denied environments. Our method uses a distributed consensus mechanism that allows robots to agree on a local coverage workspace centered around the anchor. We theoretically analyzed the correctness of the approach and extensively validated the effectiveness of the proposed AOC approach in efficiently covering the environment and have shown comparable performance to GPS-based coverage. The results provided strong evidence for the practicality of the proposed AOC in various scenarios, including target tracking and monitoring dynamic environmental features, provided these features are observable by all the robots.
\bibliographystyle{IEEEtran}
\bibliography{ref}

\end{document}